%% file: avr_arxiv.tex
\documentclass[10pt,twocolumn,letterpaper]{article}

\usepackage[pagenumbers]{wacv}

\usepackage{colortbl}
\usepackage{multirow}
\usepackage{bm}

\usepackage{pifont}
\newcommand{\cmark}{\ding{51}}
\newcommand{\xmark}{\ding{55}}
\usepackage{fontawesome5}
\definecolor{evgreen}{HTML}{5AA86F}
\definecolor{genblue}{HTML}{4A8FD4}
\definecolor{ourcoral}{HTML}{E8825A}
\newcommand{\icoDetect}{\textcolor{ourcoral}{\faSearch}}
\newcommand{\icoUser}{\textcolor{black!55}{\faUser}}
\newcommand{\icoGiven}{\textcolor{black!55}{\faHandPaper}}
\newcommand{\icoClip}{\textcolor{evgreen}{\faFilm}}
\newcommand{\icoGen}{\textcolor{genblue}{\faMagic}}

\definecolor{wacvblue}{rgb}{0.21,0.49,0.74}
\usepackage[pagebackref,breaklinks,colorlinks,allcolors=wacvblue]{hyperref}

\usepackage{amsthm}
\newtheorem{proposition}{Proposition}
\newtheorem{corollary}{Corollary}

\graphicspath{{figures/}{./}}

\def\confName{WACV}
\def\confYear{2027}

\title{Copy What Is Seen, Generate What Is Not:\\Training-Free Anomaly-Aware Video Restoration}

\author{Zhida Qu$^{1}$ \quad Shengchao Chen$^{2,}$\thanks{Corresponding author.}\\[3pt]
{\normalsize $^{1}$Department of Computer Science and Engineering, Tandon School of Engineering, New York University}\\
{\normalsize $^{2}$Australian AI Institute, University of Technology Sydney}\\
{\tt\small zq2195@nyu.edu \quad shengchao.chen.uts@gmail.com}
}

\begin{document}
\maketitle

\begin{abstract}
A surveillance system that detects an anomaly often has to repair the footage as well, yet the two tasks are studied in isolation: training-free anomaly detectors stop at a score or a label, while training-free video editing answers to a user prompt rather than to a detector. This paper proposes AVR (Anomaly-aware Video Restoration), which closes that gap with frozen pretrained models alone and generates content only where the clip offers no evidence to copy. Motion evidence first gates open-vocabulary proposals into spatio-temporal masks. A background prior computed from the clip then fills every pixel the anomaly ever uncovers, leaving diffusion to synthesize only what no frame showed, and a frozen verifier decides per clip whether to trust a classical, a prior-anchored, or a background-conditioned restorer. Extensive experiments on three surveillance datasets, under both full-reference anomaly injection and real anomalies, show that AVR leads full-frame fidelity under oracle masks, matches three trained video inpainters inside the edited region, and outperforms a detect-then-generate pipeline on the masks it produces itself, while suppressing both the residual anomaly and the flicker of free diffusion.
\end{abstract}

\section{Introduction}
\label{sec:intro}

\begin{figure*}[tbh]
\centering
\includegraphics[width=.96\textwidth]{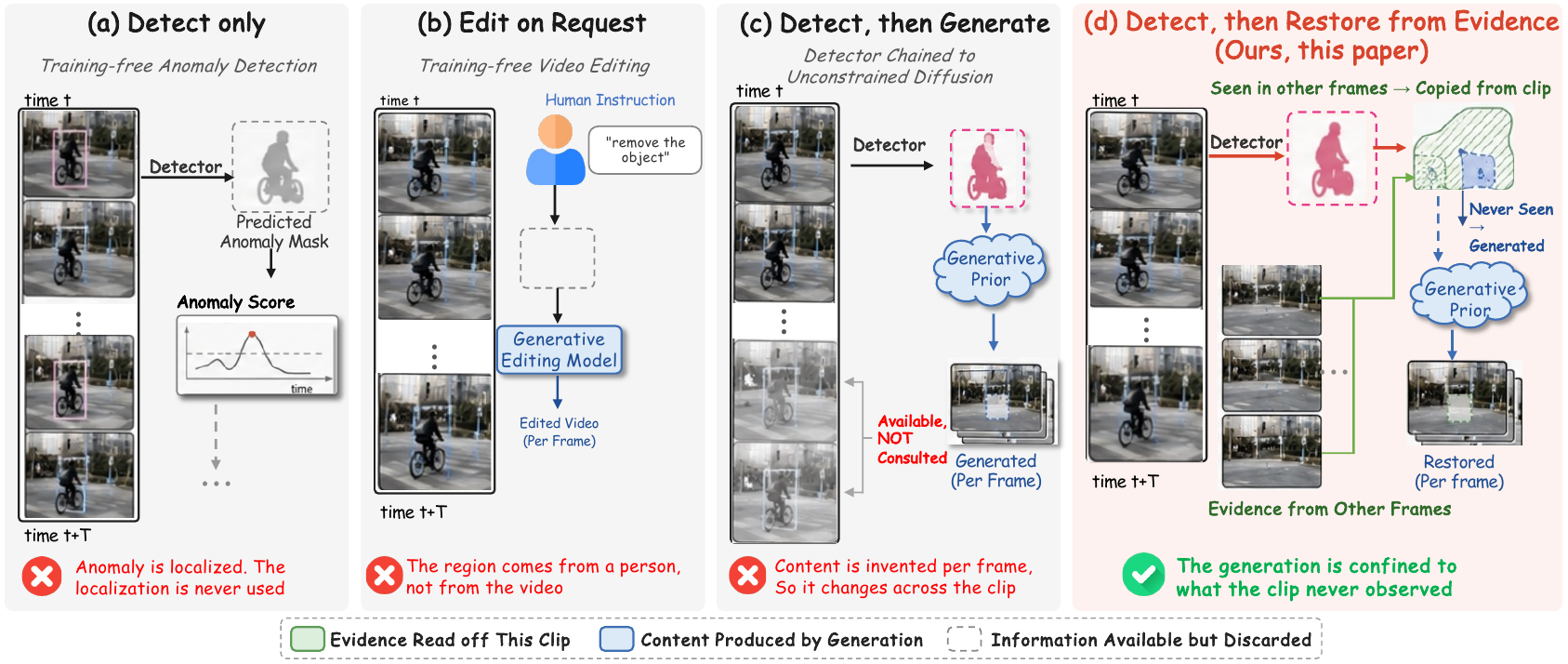}
\caption{\textbf{Comparison with existing paradigms.} (a) Detection localizes and stops, (b) editing needs a person to supply the region, and (c) chaining the two fills from a generative prior alone. (d)~AVR copies every pixel seen elsewhere and generates only what no frame observed.}
\label{fig:hero}
\end{figure*}

Video anomaly detection (VAD) is a core component of intelligent surveillance, elderly-care, and industrial-safety systems~\cite{iot_vad_survey,vad_tutorial_survey}. Over the past decade it has progressed from reconstruction- and prediction-based deep models~\cite{chen2022dynamic,chen2023tempee} to, most recently, paradigms built on vision--language and multimodal large language models~\cite{vad_10years,lavad}. In parallel, diffusion-based video generation and inpainting have reached a quality level at which detected anomalous content can, in principle, be removed and plausibly restored rather than merely flagged~\cite{text2video_zero,diffueraser}. A system that both detects and repairs anomalies must bring the two together.

Each capability already exists on its own, and only the interface between them is missing. Detection has become training-free through language-model reasoning~\cite{lavad,anomalyagent} or embedding geometry~\cite{spherevad}, and recent pipelines localize anomalies down to pixel level~\cite{tao,lavida_zs}, yet the output ends at a score or a mask while the anomaly stays in the video (Fig.~\ref{fig:hero}a). Editing is training-free as well, from per-video tuning~\cite{tuneavideo} to inference without any optimization~\cite{text2video_zero}, but every edit is steered by a prompt or a reference, so a person still decides what to change and where (Fig.~\ref{fig:hero}b). Video inpainting~\cite{propainter,e2fgvi} is closest, since flow-guided propagation already exploits frames where the background is visible, yet it starts from a mask that some other system must supply, and it treats every pixel as equally unknown.

Chaining the two components, however, fails in three ways (Fig.~\ref{fig:hero}c). First, a detector reports a score or a box rather than a region a generator can act on, so given the whole frame it rewrites parts that were never anomalous. Second, even with the right region, the generator has no reason to reproduce the background it occludes, and no constraint links what it invents in one frame to the next. Third, any repair that requires collecting footage and fine-tuning a model for the scene is unavailable in exactly the setting that motivates the task, where anomalies are rare. All three share one cause: the information that would resolve them is already present in the clip, yet each stage of the existing chain discards it before the next stage can use it.

To address these challenges, this paper proposes AVR (Anomaly-aware Video Restoration), a training-free framework that connects anomaly detection to anomaly-guided restoration using only frozen pretrained models, and that falls back on generation only where the clip provides no evidence to copy (Fig.~\ref{fig:hero}d). Specifically, an \emph{anomaly-aware localization} module gates open-vocabulary grounding by motion evidence, so proposals survive only where the video supports them. A \emph{conditional restoration} module then fills the masked regions from a temporal-median background prior and mask-conditioned diffusion, and a frozen verifier selects per clip among classical, prior-anchored and background-conditioned restorers. The two modules form an end-to-end detect-and-restore pipeline with no dataset-specific training or fine-tuning. Our contributions are:
\begin{itemize}
\itemsep0em
\item We propose AVR, which turns detected anomalies directly into restored video at inference time, with every component of it a frozen public checkpoint.
\item We design an anomaly-aware localization module that gates open-vocabulary proposals by motion evidence, with a fallback that keeps static anomalies detectable.
\item We design a conditional restoration module that anchors diffusion to a background prior and selects per clip among classical, prior-anchored and background-conditioned restorers with a single frozen, reference-free verifier.
\item Experiments on three surveillance benchmarks and real anomalies validate AVR against trained video inpainters and detect-then-generate pipelines, under both oracle masks and the detected masks AVR itself produces.
\end{itemize}

\section{Related Work}
\label{sec:related}

\paragraph{Training-Free Video Anomaly Detection.}
Classical VAD trains reconstruction- or prediction-based models on normal data~\cite{vad_10years}, from autoencoders that learn temporal regularity~\cite{hasan2016temporal} and memory modules that refuse to reconstruct unseen patterns~\cite{gong2019memae} to predictors scored by forecast error~\cite{shanghaitech}, and ranks snippets under multiple-instance supervision where only video-level labels exist~\cite{sultani2018realworld}. Recent work instead reuses frozen foundation models~\cite{chen2023prompt,chen2026fedal}: LAVAD~\cite{lavad} chains a captioner and a large language model to score anomalies without target-dataset training, AnomalyAgent~\cite{anomalyagent} brings agentic reasoning to zero-/few-shot detection beyond video, and SphereVAD~\cite{spherevad} performs geodesic inference on hypersphere embeddings, while resource-aware variants target edge deployment~\cite{iot_vad_survey,memovad}. Two recent pipelines push past scoring to pixel-level output, TAO~\cite{tao} by recasting detection as SAM2 tracking of anomalous objects and LAVIDA~\cite{lavida_zs} by training a SAM-initialized mask decoder on pseudo-anomalies so that no anomaly data is needed. Localization without target-domain labels is therefore within reach, through a tracker, a trained decoder or frozen checkpoints alone. Every one of these systems stops at finding the anomaly, however, and none of them repairs the footage it has flagged.

\paragraph{Training-Free Video Editing.}
Video editing has shifted from per-video optimization~\cite{tuneavideo} to tuning-free inference: Text2Video-Zero~\cite{text2video_zero} turns image diffusion models into zero-shot video generators, FateZero~\cite{fatezero} and TokenFlow~\cite{tokenflow} carry attention maps and diffusion features across frames, and AnyV2V~\cite{anyv2v} generalizes it to broad video-to-video tasks. All are steered by a prompt or a reference edit, so what to change, and where, still rests with a person.

\input{tables/tab_taxonomy}

\paragraph{Video Inpainting.}
Given a mask, video inpainting fills it coherently across time. Before learned models, region filling copied exemplar patches under a confidence-driven ordering~\cite{criminisi2004exemplar}, extended to space-time patches so that missing content was borrowed rather than invented~\cite{wexler2007spacetime}. Flow-guided propagation now dominates: E2FGVI~\cite{e2fgvi} makes flow completion, feature propagation and content hallucination end-to-end trainable, and ProPainter~\cite{propainter} strengthens propagation with dual-domain features and a mask-guided sparse transformer. Diffusion has since reached the same task, either by resampling a frozen denoiser inside the mask~\cite{lugmayr2022repaint} or by training for it: DiffuEraser~\cite{diffueraser} injects a propagated prior into a video diffusion model, FloED~\cite{floed} adds a flow branch and a flow-attention cache to cut the cost of doing so, and OmnimatteZero~\cite{omnimattezero} removes objects training-free by reusing a pretrained one, at a compute cost that is itself an active concern~\cite{efficient_video_diffusion}. Inference-time guidance keeps the weights frozen and instead steers the sampling trajectory, as GradPaint~\cite{gradpaint} does by backpropagating a coherence loss against the known region at every denoising step. Table~\ref{tab:taxonomy} places these methods alongside the two families above, where no prior entry both finds its own region and distinguishes the content it puts there: those that produce content receive the region from elsewhere, and those that find a region produce none. AVR does both, deriving its region from motion-gated detection and then dividing that region by evidence rather than filling it uniformly.

\section{Methodology}
\label{sec:method}

\paragraph{Problem Formulation.}
Given a video $X=\{x_1,\dots,x_n\}$ with frames $x_i\in\mathbb{R}^{H\times W\times 3}$, we seek a training-free mapping from the observed video to its restored counterpart,
\begin{equation}
X' = \mathcal{G}\big(X,\,\mathcal{D}(X)\big),
\label{eq:pipeline}
\end{equation}
where the \emph{anomaly-aware localization} module $\mathcal{D}$ outputs a spatio-temporal mask $M=\{m_1,\dots,m_n\}$, $m_i\in\{0,1\}^{H\times W}$, and the \emph{conditional restoration} module $\mathcal{G}$ replaces the masked content while preserving unmasked pixels. Both reuse frozen pretrained parameters. Two properties shape what $\mathcal{G}$ can be. It must operate \emph{without a reference}, since at inference the clean video behind the anomaly does not exist, and the region it edits is not homogeneous: writing $U=\{p:\exists\,i,\;m_i(p)=1\}$ for the pixels the anomaly ever occupies, part of $U$ is visible in some other frame of the clip and the rest is not:
\begin{equation}
U_{\text{obs}}=\{p\in U:\exists\,i,\;m_i(p)=0\},\qquad
U_{\text{gen}}=U\setminus U_{\text{obs}}.
\label{eq:partition}
\end{equation}
Only $U_{\text{gen}}$ has to be invented, since $U_{\text{obs}}$ can be recovered from observation. Restorers that propagate across frames exploit part of $U_{\text{obs}}$ implicitly, but none makes the split explicit or conditions the choice of restorer on it. The split is the exact boundary of what observation determines.
\begin{proposition}
\label{prop:evidence}
Fix the masks, let $U\neq\emptyset$, and let the unmasked observations on $U$ agree with a background that is constant in time, with each channel in $[0,r]$. Among all estimators that read only unmasked pixels, the smallest worst-case mean squared error on $U$, summed over the three channels, equals $\tfrac{3}{4}r^2\,|U_{\text{gen}}|/|U|$, and it is attained by copying each pixel of $U_{\text{obs}}$ from a frame where it is unmasked and returning the midpoint of the range on $U_{\text{gen}}$.
\end{proposition}
\begin{proof}
The estimator is exact on $U_{\text{obs}}$, since every unmasked frame reports the same background there, and on $U_{\text{gen}}$ each channel lies in $[0,r]$, so the midpoint is off by at most $r/2$ and the squared error per pixel is at most $\tfrac34 r^2$. For the lower bound, fix any estimator and let the backgrounds $b^{(0)}$ and $b^{(1)}$ agree with the observation on $U_{\text{obs}}$ while taking the values $0$ and $r$ on every channel of $U_{\text{gen}}$. Both induce the same observation, so the estimator returns a single $\hat b$ for the two, and $t^2+(t-r)^2\ge r^2/2$ gives
\begin{equation*}
R(\hat b,b^{(0)})+R(\hat b,b^{(1)})\;\ge\;\tfrac{3}{2}r^2\,|U_{\text{gen}}|/|U|,
\end{equation*}
so the worse of the two matches the upper bound. Both degenerate cases follow at once, since one of the two sums over $U$ is then empty and the argument is unchanged.
\end{proof}
A copy is therefore provably exact on $U_{\text{obs}}$, whereas on $U_{\text{gen}}$ a prior is all that remains, and the ratio itself is fixed by the anomaly rather than by any choice of restorer.
\begin{corollary}
\label{cor:window}
Enlarging the set of frames leaves $|U_{\text{gen}}|/|U|$ non-increasing, so the bound of Prop.~\ref{prop:evidence} never worsens with a longer clip. If an object occludes the same pixels throughout the window, however, the ratio stays at one for every sub-window, and a longer clip changes nothing.
\end{corollary}
Any frame that leaves a pixel unmasked moves it from $U_{\text{gen}}$ to $U_{\text{obs}}$ and never back, so the ratio can only fall, whereas an object that is never once uncovered leaves $U_{\text{obs}}$ empty for every sub-window of the given clip.

\begin{figure*}[tbh]
\centering
\includegraphics[width=0.94\textwidth]{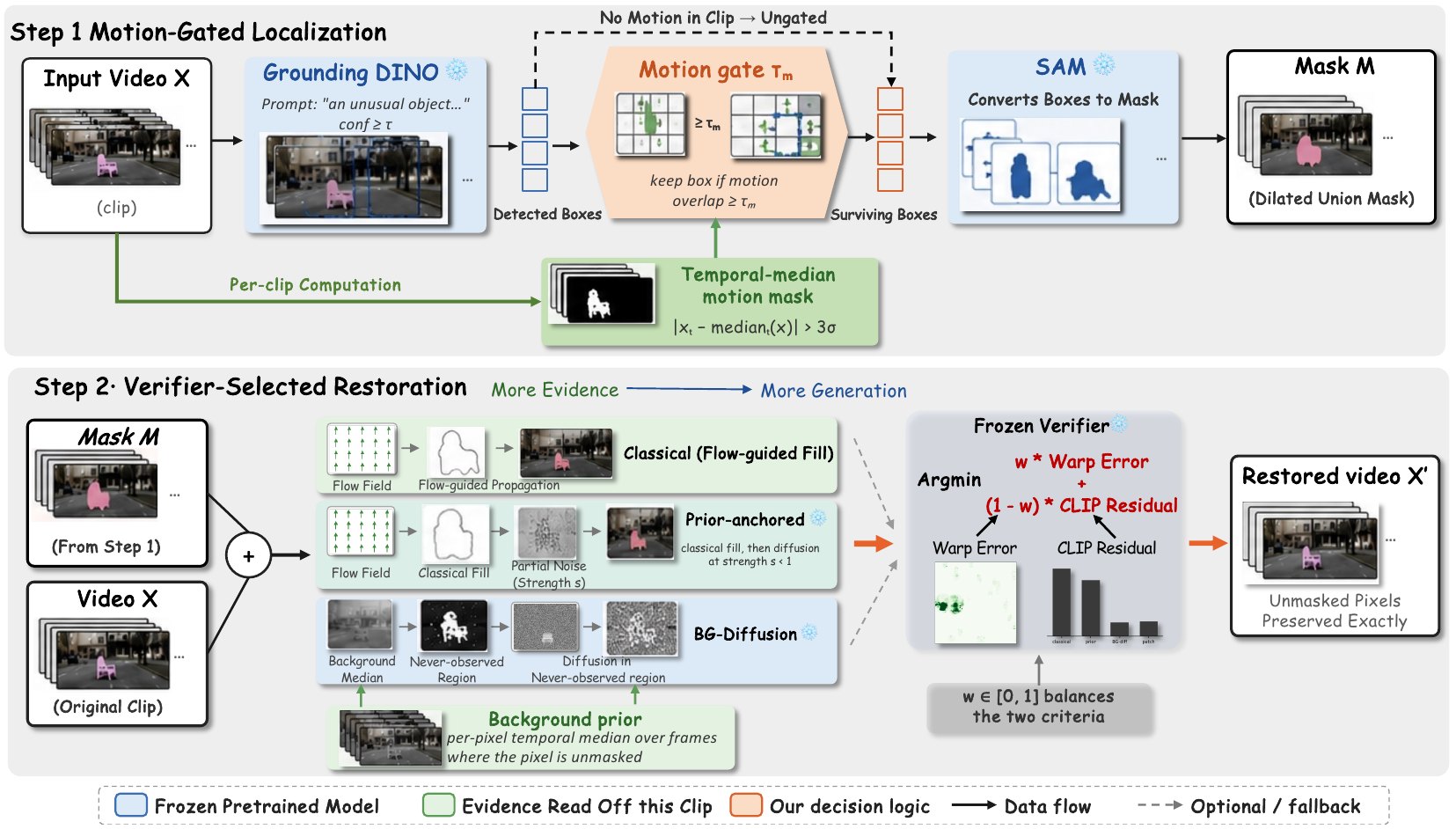}
\caption{\textbf{The AVR pipeline.} Step 1 turns a detection into the mask $M$, keeping open-vocabulary boxes only where motion measured on the clip corroborates them. Step 2 turns $M$ back into video, running three restorers ordered from copying observed pixels to synthesising never-observed ones and letting a frozen verifier pick one per clip. Pixels outside $M$ are untouched, and no component is trained.}
\label{fig:overview}
\vspace{-6pt}
\end{figure*}
\paragraph{Overview.}
Fig.~\ref{fig:overview} shows the workflow. Stage one grounds an anomaly prompt on every frame and keeps only the proposals that motion evidence supports, which turns a detector output into the spatio-temporal mask $M$. Stage two computes a background prior from the clip itself, fills $M$ from that prior wherever the clip ever exposed the pixel and from diffusion on the remainder that Prop.~\ref{prop:evidence} shows no estimator can recover, blends each restored frame against its flow-warped predecessor, and lets a frozen verifier choose among the candidates. Neither stage fits a parameter to the target scene, so the only quantity that changes from clip to clip is the evidence that is contained in the clip itself.

\paragraph{Anomaly-Aware Localization.}
Anomalies are open-ended by definition, so the detector cannot be a closed-set classifier. For each frame $x_i$, an open-vocabulary detector (Grounding DINO~\cite{groundingdino}) grounds a fixed anomaly prompt $P$ and keeps boxes above a confidence threshold $\tau$, $\mathcal{B}_i=\{b:\operatorname{conf}(b)\ge\tau\}$. A promptable segmenter (SAM~\cite{sam}) converts the boxes into pixel masks, whose union, after dilation $\rho_\delta$, forms the frame mask
\begin{equation}
m_i = \rho_\delta\Big(\textstyle\bigcup_{b\in \mathcal{B}_i}\operatorname{SAM}(x_i,b)\Big),
\label{eq:mask}
\end{equation}
with $m_i=\mathbf{0}$ when no box survives. Open-vocabulary grounding alone is noisy on surveillance footage, firing on benign scene structure. We therefore \emph{gate proposals by motion evidence}: a temporal-median background subtraction marks every pixel whose deviation exceeds the frame's mean deviation by $\kappa$ standard deviations, and a box is kept only if at least a fraction $\tau_m$ of its area is covered by that motion mask. A clip with no motion evidence anywhere falls back to ungated detection, so gating can shrink a mask but never empty one. Gating keeps the promptable, open-world interface of grounding intact, while discarding each proposal that the video itself does not support.

\paragraph{Background-Conditioned Restoration.}
The clip itself provides a strong prior for what the occluded background should look like. We form a background image $\bar{B}$ by taking, at every pixel, the temporal median over the frames where that pixel is \emph{not} masked. Writing $u$ for the indicator of $U_{\text{gen}}$ in Eq.~\eqref{eq:partition}, the pixels observed in no frame at all, the restoration composites three sources:
\begin{equation}
\hat{x}_i = (1-m_i)\odot x_i + m_i\odot\big[(1-u)\odot\bar{B} + u\odot \mathcal{G}_{\mathrm{diff}}(c_i, u)\big],
\label{eq:restore}
\end{equation}
where $c_i=(1-m_i)\odot x_i+m_i\odot\bar{B}$ is the frame with the prior already composited into the mask and $\mathcal{G}_{\mathrm{diff}}$ is a frozen inpainting diffusion model (Stable Diffusion inpainting, built on latent diffusion~\cite{ldm}) run under a neutral background prompt. Every pixel the clip ever exposes is copied from $\bar{B}$, and open-ended generation is invoked only on the set where, by Prop.~\ref{prop:evidence}, no estimator could do better.

\input{tables/tab_restoration}

\paragraph{Prior-Anchored Refinement.}
Copying alone fails when the exposed background is only partly observed, and free generation ignores what was observed. A third restorer sits between the extremes: we fill the mask with a flow-guided classical operator in the spirit of exemplar-based completion~\cite{criminisi2004exemplar}, then run diffusion on that fill at strength $s<1$, so sampling starts from a partially noised classical fill rather than pure noise~\cite{sdedit}. Diffusion then \emph{refines} an existing estimate instead of inventing one. DiffuEraser~\cite{diffueraser} learns the same anchoring, whereas frozen weights leave this one candidate among several rather than the fixed answer.

\paragraph{Verifier-Guided Selection.}
No single restorer wins on every clip: the classical operator is safest under clean motion, the background prior when the scene is mostly static, and diffusion where no observation exists. Rather than fixing this choice at design time, we run all three candidates and let a frozen verifier pick one per clip. We refer to the three as Classical (the flow-guided fill alone), Prior-anchored (that fill refined by diffusion) and BG-diffusion (the background-conditioned fill of Eq.~\eqref{eq:restore}). The verifier is reference-free, selecting over the pool with two no-reference signals scored inside the mask:
\begin{equation}
k^\star=\arg\min_{k}\; w\,\widetilde{E}_{\mathrm{warp}}^{(k)}+(1-w)\,\widetilde{R}_{\mathrm{CLIP}}^{(k)},
\label{eq:verifier}
\end{equation}
where $k$ indexes the pool, $w\in[0,1]$ balances flow warping error against the mean CLIP~\cite{clip} anomaly residual inside the mask, and $\widetilde{\cdot\,}$ denotes min--max normalization across the pool. The cost is paid at inference rather than in training~\cite{inference_scaling}, and neither the candidates nor the verifier has any parameter that is fitted on the target data at any stage.

\paragraph{Temporal Consistency.}
Restoring each frame on its own leaves the mask free to change content between neighbours, which reads as flicker. Two mechanisms enforce coherence. First, all frames share a common diffusion noise seed, a lightweight analogue of the correlated noise used in zero-shot video generation~\cite{text2video_zero}. Second, each restored frame is blended, inside the mask, with its predecessor warped forward by dense optical flow $\mathcal{W}_{i-1\rightarrow i}$ (Farneb\"ack):
\begin{equation}
x'_i = m_i \odot\big[\lambda\,\hat{x}_i + (1-\lambda)\,\mathcal{W}_{i-1\rightarrow i}(x'_{i-1})\big] + (1-m_i)\odot x_i,
\label{eq:temporal}
\end{equation}
for $i\ge 2$, with $x'_1=\hat{x}_1$. The blend weight $\lambda$ trades spatial fidelity to the current restoration against temporal fidelity to the propagated history, and setting $\lambda{=}1$ removes the blend altogether and restores every frame independently.

\section{Experiments and Results}
\label{sec:exp}

\paragraph{Implementation Details.}
We adopt public checkpoints throughout and keep every one of them frozen. Localization runs Grounding DINO-tiny~\cite{groundingdino} under the fixed prompt ``\emph{an unusual object.\ an anomaly.\ a foreign object.}'' followed by SAM ViT-B~\cite{sam} or SAM2-hiera-large~\cite{sam2}, and restoration runs Stable Diffusion v1.5 inpainting~\cite{ldm} at $512^2$ with 25 denoising steps and guidance 7.5. A single Farneb\"ack estimator~\cite{farneback2003} supplies all optical flow, serving the temporal blend, the warping error and the flow-guided fill alike, while the verifier scores candidates with CLIP ViT-B/32~\cite{clip}. We fix $\tau{=}0.25$, $\tau_m{=}0.35$, $\kappa{=}3$, $s{=}0.7$, $w{=}0.5$ and $\lambda{=}0.5$, dilate masks by $\delta{=}5$ pixels after detection and $7$ before inpainting, and implement everything in PyTorch on two NVIDIA RTX A5500 GPUs each with 24\,GB of memory.

\paragraph{Benchmarks and Metrics.}
Real anomalies come with no clean reference, so we build anomaly-injection benchmarks from the normal portions of CUHK Avenue~\cite{avenue}, UCSD Ped2~\cite{ped2} and ShanghaiTech~\cite{shanghaitech}, each holding 60 non-overlapping clips (16 frames, $256^2$) split evenly over a static object, a moving object and a local appearance flicker, with the pre-injection clip as reference. \emph{Fidelity}: full-frame PSNR, SSIM and LPIPS~\cite{lpips}, plus PSNR$_{\text{reg}}$ inside the ground-truth region, which covers only ${\sim}0.5$--$1.2\%$ of the frame. \emph{Temporal behaviour}, which framewise fidelity misses~\cite{szeto2022devil}: tLP, Farneb\"ack warping error, and FVD~\cite{fvd} on R3D-18 Kinetics features, shared by every FVD number here and so not comparable to published I3D values. \emph{Anomaly removal}: Mask IoU, and residual drop (RD), which uses no text prompt and no cropping. Each video supplies its own temporal-median reference, both it and every frame are cut into a $4\times4$ grid and embedded with frozen CLIP ViT-B/32~\cite{clip}, and a cell scores one minus the cosine similarity. With $a(V)$ the mean of that map inside the region, RD is the relative reduction $\big(a(X)-a(X')\big)/a(X)$. Rescoring with DINO preserves every ranking and sign, so RD does not rest on the verifier's backbone. We also run 20 anomalous ShanghaiTech \emph{test} clips, which injection cannot cover, under the no-reference measures alone.

\input{tables/tab_pipeline}
\paragraph{Restoration Quality.}
Table~\ref{tab:main_inpainting} compares restorers under oracle masks, so the trained inpainters ProPainter, E2FGVI and FloED~\cite{floed} differ from AVR only in how they fill the hole. AVR attains the best full-frame PSNR and FVD on all three benchmarks, and inside the region it leads on ShanghaiTech and stays within a decibel of the strongest trained inpainter elsewhere. Anchoring the fill to the clip's own observed background therefore matches dedicated training on the pixels it edits. DiffuEraser and OmnimatteZero trail on every metric, exhibiting exactly the content hallucination that this evaluation protocol is built to expose.

\begin{figure}[tbh]
\centering
\includegraphics[width=1\columnwidth]{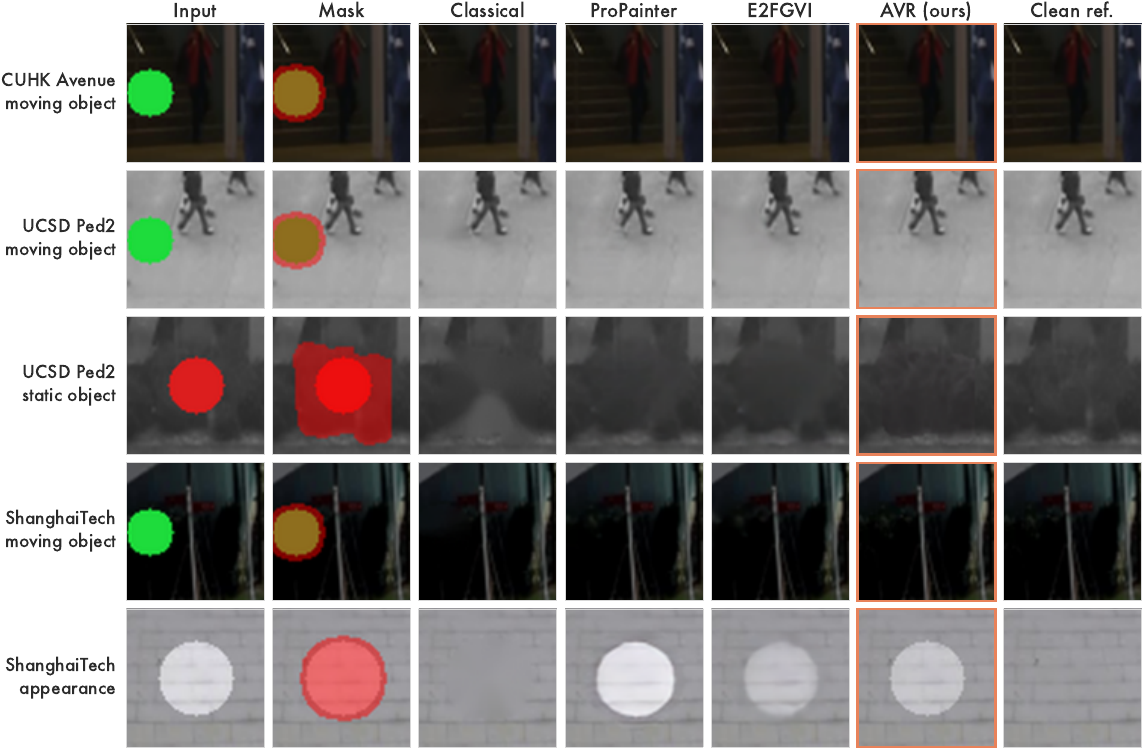}
\caption{End-to-end restoration under detected masks. All restorers share the same motion-gated masks, one median clip per row.}
\label{fig:qual_e2e}
\vspace{-14pt}
\end{figure}

\paragraph{End-to-End Comparison.}
Table~\ref{tab:main_pipeline} evaluates the full pipeline with detection included, and Fig.~\ref{fig:qual_e2e} shows the same masks and restorers clip by clip. AVR gains 17.9\,dB over the Grounded-Stable-Diffusion (Grounded-SD) base on ShanghaiTech, turns the residual drop positive, and leads ProPainter, E2FGVI and FloED on PSNR and SSIM on every benchmark while alone producing the masks it is scored on. The two ProPainter rows attribute that gain, since motion-gated masks by themselves lift its ShanghaiTech PSNR by 13\,dB. Restoration quality therefore follows the evidence that a given mask exposes rather than the raw capacity of whichever restorer happens to fill it.

\input{tables/tab_ablation}
\begin{figure}[tbh]
\centering
\includegraphics[width=.93\columnwidth]{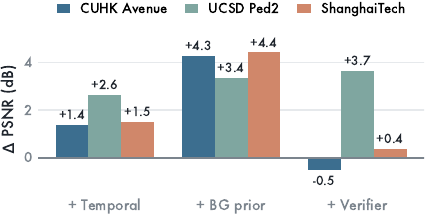}
\caption{PSNR that each mechanism adds to the preceding stage under oracle masks, from per-frame diffusion onward.}
\label{fig:temporal}
\vspace{-7pt}
\end{figure}

\begin{figure*}[tbh]
\centering
\includegraphics[width=0.9\textwidth]{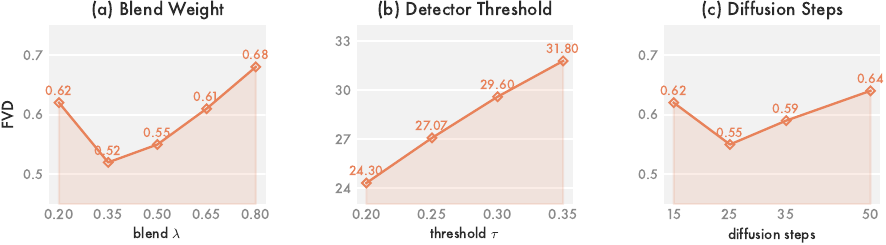}
\caption{FVD sensitivity on Avenue to blend $\lambda$, detector threshold $\tau$ and diffusion steps. Panel (b) is measured end to end while (a) and (c) use oracle masks, which puts its FVD an order of magnitude higher.}
\label{fig:hparam}
\vspace{-8pt}
\end{figure*}

\paragraph{Ablations.}
Table~\ref{tab:ablation} traces the end-to-end gains to their two sources. Motion gating carries fidelity, raising PSNR by 5.2--15.3\,dB over the SAM2 detector, while verifier selection acts on removal, turning ShanghaiTech RD positive at nearly unchanged PSNR. The two are therefore complementary rather than redundant. Under oracle masks, Fig.~\ref{fig:temporal} separates the restorer's stages, and the background prior is the largest contributor on every benchmark.

\paragraph{Detector and Hyperparameters.}
Table~\ref{tab:detection} varies the detector with the restorer held fixed, including TAO~\cite{tao}, which tracks anomalous objects into a mask. Motion gating leads every alternative on every metric on both benchmarks, so demanding independent motion evidence behind each semantic proposal beats open-vocabulary grounding and tracking alike. Fig.~\ref{fig:hparam} sweeps $\lambda$, $\tau$ and the diffusion steps, where FVD is minimized near $\lambda\approx0.35$ while the $\tau$ trend reverses across the three benchmarks, so we keep one shared setting across every run and benchmark.

\input{tables/tab_detector}
\paragraph{Selection Analysis.}
Fig.~\ref{fig:select} shows how the verifier distributes its choices, and the split tracks mask quality rather than dataset. Under oracle masks it picks Classical or Prior-anchored on 35--56\% of clips, while under detected masks, which may cover clean background, it retreats to BG-diffusion on 87--93\%. The pool is therefore not redundant, and the verifier turns conservative exactly when localization is unreliable. Against a per-clip oracle it recovers most of the remaining headroom on Ped2, where it beats every fixed candidate, and a smaller share on ShanghaiTech, while on Avenue one candidate dominates nearly every clip and leaves little to select. The last row of Fig.~\ref{fig:qual_e2e} shows the residual failure mode, where the classical candidate clears an appearance change that the verifier declines.

\begin{figure}[tbh]
\centering
\includegraphics[width=0.86\columnwidth]{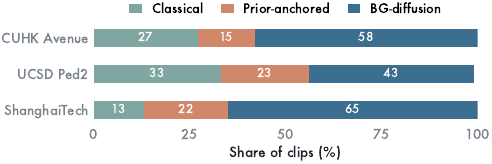}
\caption{Restorer chosen by the frozen verifier, as a share of the 60 oracle-mask clips available in each benchmark.}
\label{fig:select}
\vspace{-8pt}
\end{figure}

\paragraph{Efficiency.}
Table~\ref{tab:efficiency} reports inference cost per 16-frame clip and per frame, with the pipeline split into stages and totals under the same masks. The two diffusion candidates dominate, the verifier adds under four seconds, and AVR runs at roughly 2.4 times its single-restorer variant, while ProPainter is about eight times faster. This is the price of selection, paid at inference rather than training.

\input{tables/tab_efficiency}
\paragraph{Real Anomalies.}
Table~\ref{tab:real} leaves the injection protocol entirely. On 20 genuine ShanghaiTech anomalies AVR cuts flicker in the removed region more than six-fold, brings warping error below the input video's own, and reaches a positive residual drop, while the base degrades on all three counts. An evidence-anchored fill therefore generalizes to anomalies the protocol never synthesized. Re-running the detector over the restored clips agrees, with the fraction of frames still firing dropping from 1.00 to 0.42.
\input{tables/tab_real}

\paragraph{Failure Analysis.}
Fig.~\ref{fig:bytype} splits RD by anomaly type and locates every negative number in the paper. Moving objects and local appearance changes are removed cleanly on all benchmarks, while static objects are negative everywhere. The cause is structural: a static object occludes the same pixels in every frame, so $U_{\text{gen}}$ grows to fill $U$ in Eq.~\eqref{eq:partition} and Eq.~\eqref{eq:restore} degenerates to unconstrained generation, which is Prop.~\ref{prop:evidence} at $|U_{\text{gen}}|/|U|=1$. Detection shows the same asymmetry, since most clips with no surviving proposal are static objects, which supply neither the motion evidence that admits them nor the observations to repair them.

\begin{figure}[tbh]
\centering
\includegraphics[width=0.95\columnwidth]{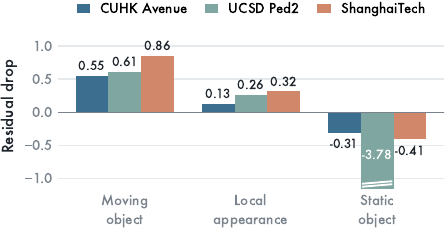}
\caption{Residual drop by anomaly type, end to end on all three benchmarks, with the axis truncated at $-1.15$.}
\label{fig:bytype}
\vspace{-8pt}
\end{figure}

\paragraph{Case Study.}
Fig.~\ref{fig:caseadv} places AVR beside ProPainter on identical detected masks with a clean reference. A moving object exposes its background in other frames and the prior recovers it outright, an appearance change is only partly reversed by either restorer, and a static object by neither, since half of it falls outside the mask and none is ever uncovered. Fig.~\ref{fig:qual} turns to real anomalies, where per-frame diffusion removes a cyclist but invents a pedestrian in one frame and a potted plant twelve later, while AVR keeps the walkway stable. Each hallucination is locally plausible, the failure full-frame metrics miss and region scoring catches.

\begin{figure}[tbh]
\centering
\includegraphics[width=1\columnwidth]{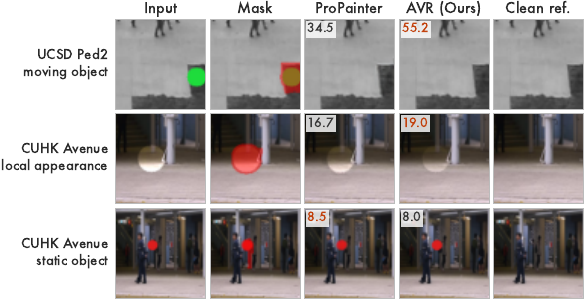}
\caption{Case study on injected anomalies, with ProPainter and AVR restoring identical detected masks. Numbers give the region PSNR in decibels obtained on that one particular clip.}
\label{fig:caseadv}
\end{figure}

\begin{figure}[tbh]
\centering
\includegraphics[width=.9\columnwidth]{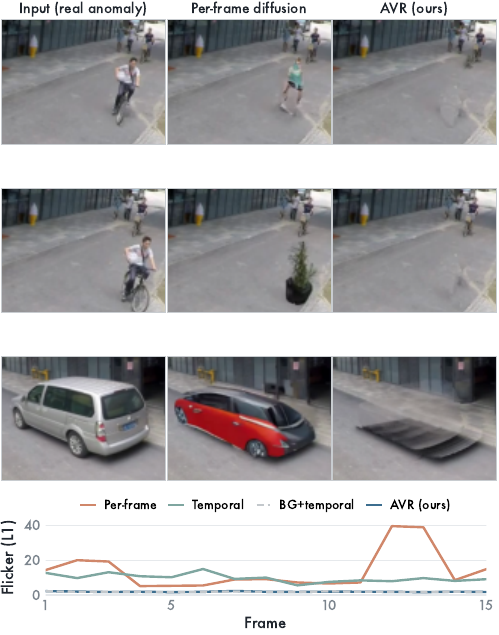}
\caption{Real anomalies under detected masks. Rows 1--2 show a cyclist with its flicker curve, and row 3 a \emph{parked} car.}
\label{fig:qual}
\end{figure}

\section{Conclusion}
\label{sec:conclusion}

This paper proposes AVR, a training-free framework that couples anomaly detection to restoration and generates content only where the clip offers no evidence to copy. A localization module gates open-vocabulary proposals by motion evidence, and a restoration module fills the mask from a background prior, leaving a frozen verifier to decide how much remains to be generated. Extensive results on three surveillance benchmarks and on real anomalies demonstrate both the effectiveness and the superiority of AVR.

{
    \small
    \bibliographystyle{ieeenat_fullname}
    \bibliography{refs}
}

\end{document}

%% file: tables/tab_taxonomy.tex
\begin{table}[t]
\centering
\caption{Comparison of method families along the two inputs a detect-and-restore system requires. \emph{Content} follows the colour code of Fig.~\ref{fig:hero}: \icoClip~pixels copied from other frames of the same clip, \icoGen~pixels produced by a generative prior rather than observed.}
\label{tab:taxonomy}
\setlength{\tabcolsep}{4.5pt}
\renewcommand{\arraystretch}{1.15}
\resizebox{\columnwidth}{!}{%
\begin{tabular}{@{}ll cc c l@{}}
\toprule
& Method & Region & Content & Train & Temporal model \\
\midrule
\multirow{2}{*}{\rotatebox{90}{\scriptsize Detect}}
 & LAVAD~\cite{lavad}         & \icoDetect~Detector & ---            & \xmark & --- \\
 & SphereVAD~\cite{spherevad} & \icoDetect~Detector & ---            & \xmark & --- \\
\cmidrule(l{0pt}r{0pt}){2-6}
\multirow{2}{*}{\rotatebox{90}{\scriptsize Edit}}
 & T2V-Zero~\cite{text2video_zero} & \icoUser~User & \icoGen          & \xmark & Attention \\
 & TokenFlow~\cite{tokenflow}      & \icoUser~User & \icoGen          & \xmark & Feature prop. \\
\cmidrule(l{0pt}r{0pt}){2-6}
\multirow{4}{*}{\rotatebox{90}{\scriptsize Inpaint}}
 & E2FGVI~\cite{e2fgvi}              & \icoGiven~Given & \icoClip\,\icoGen & \cmark & Learned prop. \\
 & ProPainter~\cite{propainter}      & \icoGiven~Given & \icoClip\,\icoGen & \cmark & Learned prop. \\
 & DiffuEraser~\cite{diffueraser}    & \icoGiven~Given & \icoClip\,\icoGen & \cmark & Learned prop. \\
 & OmnimatteZero~\cite{omnimattezero}& \icoGiven~Given & \icoGen          & \xmark & Layer decomp. \\
\midrule
\rowcolor{gray!12}
 & \textbf{AVR (Ours)} & \icoDetect~\textbf{Detector} & \icoClip$\rightarrow$\icoGen & \xmark & Noise $+$ flow blend \\
\bottomrule
\end{tabular}%
}
\vspace{-4pt}
\end{table}

%% file: tables/tab_restoration.tex
\begin{table*}[tbh]
\centering
\caption{Restoration under oracle masks. PSNR$_{\text{reg}}$ is measured inside the anomaly region, \textbf{bold} for best and \underline{underlined} for second best.}
\label{tab:main_inpainting}
\resizebox{.95\textwidth}{!}{%
\begin{tabular}{l cccc cccc cccc}
\toprule
 & \multicolumn{4}{c}{CUHK Avenue} & \multicolumn{4}{c}{UCSD Ped2} & \multicolumn{4}{c}{ShanghaiTech} \\
\cmidrule(lr){2-5} \cmidrule(lr){6-9} \cmidrule(lr){10-13}
Method & PSNR$\uparrow$ & PSNR$_{\text{reg}}\uparrow$ & LPIPS$\downarrow$ & FVD$\downarrow$ & PSNR$\uparrow$ & PSNR$_{\text{reg}}\uparrow$ & LPIPS$\downarrow$ & FVD$\downarrow$ & PSNR$\uparrow$ & PSNR$_{\text{reg}}\uparrow$ & LPIPS$\downarrow$ & FVD$\downarrow$ \\
\midrule
Classical (Spatial) & 39.76 & 18.79 & 0.008 & 5.86 & 48.84 & 27.96 & \underline{0.004} & 2.78 & 52.09 & 31.22 & 0.004 & 1.97 \\
Classical (Flow) & 40.36 & 19.50 & 0.007 & 6.12 & 49.31 & 28.61 & \underline{0.004} & 2.51 & 51.88 & 31.26 & 0.004 & 2.30 \\
ProPainter~\cite{propainter} & 43.73 & 25.34 & \textbf{0.004} & 0.64 & 50.56 & 32.08 & \underline{0.004} & \underline{0.49} & 53.88 & 35.21 & \underline{0.003} & \underline{0.64} \\
E2FGVI~\cite{e2fgvi} & \underline{45.97} & \textbf{27.34} & \textbf{0.004} & \underline{0.57} & \underline{51.90} & \textbf{33.33} & \underline{0.004} & 0.53 & \underline{55.26} & \underline{36.41} & 0.004 & 0.67 \\
FloED~\cite{floed} & 45.39 & 26.19 & \underline{0.005} & 0.77 & 51.00 & 32.47 & \underline{0.004} & 0.51 & 54.35 & 35.95 & 0.004 & 0.76 \\
DiffuEraser~\cite{diffueraser} & 31.98 & 12.40 & 0.037 & 8.98 & 34.28 & 14.01 & 0.056 & 15.64 & 36.22 & 16.91 & 0.084 & 20.22 \\
OmnimatteZero~\cite{omnimattezero} & 31.74 & 11.31 & 0.038 & 12.18 & 32.24 & 12.13 & 0.050 & 15.95 & 35.23 & 15.72 & 0.051 & 18.62 \\
\midrule
\rowcolor{gray!12}
AVR (Ours) & \textbf{47.27} & \underline{27.17} & \textbf{0.004} & \textbf{0.55} & \textbf{52.63} & \underline{32.61} & \textbf{0.003} & \textbf{0.45} & \textbf{56.98} & \textbf{37.83} & \textbf{0.002} & \textbf{0.56} \\
\bottomrule
\end{tabular}%
}
\end{table*}

%% file: tables/tab_pipeline.tex
\begin{table*}[tbh]
\centering
\caption{End-to-end comparison. The \emph{base} row chains a per-frame detector with unconstrained inpainting, while ProPainter, E2FGVI and FloED restore the same motion-gated masks as AVR, hence identical Mask IoU. Best per column in \textbf{bold}, second best \underline{underlined}.}
\label{tab:main_pipeline}
\resizebox{.95\textwidth}{!}{%
\begin{tabular}{ll ccccccccc}
\toprule
Benchmark & Method & PSNR$\uparrow$ & PSNR$_{\text{reg}}\uparrow$ & SSIM$\uparrow$ & LPIPS$\downarrow$ & tLP$\downarrow$ & FVD$\downarrow$ & Warp$\downarrow$ & RD$\uparrow$ & IoU$\uparrow$ \\
\midrule
\multirow{6}{*}{CUHK Avenue} & Grounded-SD (base) & 29.00 & \textbf{15.78} & 0.972 & 0.048 & 0.036 & 46.13 & 1.93 & $-$1.202 & 0.097 \\
 & ProPainter (SAM2 masks) & 28.63 & 13.38 & 0.974 & 0.057 & 0.017 & 51.32 & 1.18 & $-$\underline{0.117} & 0.095 \\
 & ProPainter (gated masks) & 32.52 & 14.63 & 0.986 & 0.035 & \textbf{0.012} & 28.33 & \textbf{1.09} & $-$0.201 & \textbf{0.205} \\
 & E2FGVI (gated masks) & \underline{33.34} & \underline{15.05} & \underline{0.987} & \underline{0.033} & 0.014 & \underline{28.16} & 1.13 & $-$0.238 & \textbf{0.205} \\
 & FloED (gated masks) & 32.72 & 14.89 & 0.986 & 0.038 & 0.016 & 29.97 & 1.15 & $-$0.271 & \textbf{0.205} \\
 & \cellcolor{gray!12}AVR (Ours) & \cellcolor{gray!12}\textbf{33.79} & \cellcolor{gray!12}15.00 & \cellcolor{gray!12}\textbf{0.988} & \cellcolor{gray!12}\textbf{0.031} & \cellcolor{gray!12}\underline{0.013} & \cellcolor{gray!12}\textbf{27.07} & \cellcolor{gray!12}\underline{1.10} & \cellcolor{gray!12}$\bm{+}$\textbf{0.123} & \cellcolor{gray!12}\textbf{0.205} \\
\midrule
\multirow{6}{*}{UCSD Ped2} & Grounded-SD (base) & 33.25 & 21.53 & 0.984 & 0.022 & 0.043 & 22.38 & 1.64 & $-$6.796 & 0.288 \\
 & ProPainter (SAM2 masks) & 36.54 & 21.60 & 0.960 & 0.075 & 0.025 & 30.80 & 1.14 & $-$1.256 & 0.299 \\
 & ProPainter (gated masks) & 41.80 & 25.45 & \underline{0.995} & \underline{0.015} & \textbf{0.020} & \underline{6.04} & \underline{1.04} & $\bm{-}$\textbf{0.685} & \textbf{0.375} \\
 & E2FGVI (gated masks) & \underline{43.52} & \textbf{26.73} & \underline{0.995} & 0.016 & 0.023 & 6.70 & 1.08 & $-$\underline{0.763} & \textbf{0.375} \\
 & FloED (gated masks) & 41.96 & 25.81 & \underline{0.995} & 0.018 & 0.026 & 6.34 & 1.11 & $-$0.806 & \textbf{0.375} \\
 & \cellcolor{gray!12}AVR (Ours) & \cellcolor{gray!12}\textbf{44.00} & \cellcolor{gray!12}\underline{25.89} & \cellcolor{gray!12}\textbf{0.996} & \cellcolor{gray!12}\textbf{0.013} & \cellcolor{gray!12}\underline{0.021} & \cellcolor{gray!12}\textbf{5.86} & \cellcolor{gray!12}\textbf{1.03} & \cellcolor{gray!12}$-$0.970 & \cellcolor{gray!12}\textbf{0.375} \\
\midrule
\multirow{6}{*}{ShanghaiTech} & Grounded-SD (base) & 22.23 & 19.32 & 0.715 & 0.381 & 0.401 & 134.09 & 23.36 & $-$14.383 & 0.119 \\
 & ProPainter (SAM2 masks) & 25.50 & 16.48 & 0.831 & 0.320 & 0.031 & 104.61 & 2.17 & $-$0.648 & 0.135 \\
 & ProPainter (gated masks) & 38.53 & 20.88 & \underline{0.993} & \textbf{0.036} & \textbf{0.013} & \textbf{17.14} & \textbf{0.47} & $-$0.271 & \textbf{0.343} \\
 & E2FGVI (gated masks) & \underline{39.06} & \underline{21.94} & \underline{0.993} & \underline{0.037} & \underline{0.015} & 18.26 & 0.51 & $-$\underline{0.122} & \textbf{0.343} \\
 & FloED (gated masks) & 38.62 & 21.70 & \underline{0.993} & 0.040 & 0.016 & 18.41 & \underline{0.49} & $-$0.297 & \textbf{0.343} \\
 & \cellcolor{gray!12}AVR (Ours) & \cellcolor{gray!12}\textbf{40.08} & \cellcolor{gray!12}\textbf{22.91} & \cellcolor{gray!12}\textbf{0.994} & \cellcolor{gray!12}\textbf{0.036} & \cellcolor{gray!12}\textbf{0.013} & \cellcolor{gray!12}\underline{17.37} & \cellcolor{gray!12}\textbf{0.47} & \cellcolor{gray!12}$\bm{+}$\textbf{0.255} & \cellcolor{gray!12}\textbf{0.343} \\
\bottomrule
\end{tabular}%
}
\vspace{-6pt}
\end{table*}

%% file: tables/tab_ablation.tex
\begin{table}[tbh]
\centering
\caption{Ablation study. The SAM2 rows predate the background prior, so their pool holds two restorers instead of three. \textbf{Bold}: the best result, and \underline{underlined}: the second best.}
\label{tab:ablation}
\scriptsize
\setlength{\tabcolsep}{2.4pt}
\renewcommand{\arraystretch}{1.04}
\resizebox{\columnwidth}{!}{%
\begin{tabular}{ll cc cc cc}
\toprule
 & & \multicolumn{2}{c}{CUHK Avenue} & \multicolumn{2}{c}{UCSD Ped2} & \multicolumn{2}{c}{ShanghaiTech} \\
\cmidrule(lr){3-4} \cmidrule(lr){5-6} \cmidrule(lr){7-8}
Detector & Restorer & PSNR$\uparrow$ & RD$\uparrow$ & PSNR$\uparrow$ & RD$\uparrow$ & PSNR$\uparrow$ & RD$\uparrow$ \\
\midrule
SAM2 & Prior-anchored & 27.84 & $-$0.255 & 34.86 & $-$2.550 & 23.90 & $-$1.413 \\
SAM2 & Selection (2 cand.) & 28.59 & $-$0.071 & 37.25 & $-$\underline{1.220} & 24.79 & $-$0.870 \\
Motion-gated & BG-diffusion & \underline{33.19} & $\bm{+}$\textbf{0.145} & \underline{43.23} & $-$4.282 & \underline{39.74} & $-$\underline{0.072} \\
\rowcolor{gray!12}
Motion-gated & Selection (3 cand.) & \textbf{33.79} & $+$\underline{0.123} & \textbf{44.00} & $\bm{-}$\textbf{0.970} & \textbf{40.08} & $\bm{+}$\textbf{0.255} \\
\bottomrule
\end{tabular}%
}
\vspace{-8pt}
\end{table}

%% file: tables/tab_detector.tex
\begin{table}[tbh]
\centering
\caption{Detector study with the restoration module held fixed at BG-diffusion. Best per column in \textbf{bold}, second best \underline{underlined}.}
\label{tab:detection}
\scriptsize
\setlength{\tabcolsep}{2.2pt}
\renewcommand{\arraystretch}{1.04}
\resizebox{\columnwidth}{!}{%
\begin{tabular}{l cccc cccc}
\toprule
 & \multicolumn{4}{c}{CUHK Avenue} & \multicolumn{4}{c}{ShanghaiTech} \\
\cmidrule(lr){2-5} \cmidrule(lr){6-9}
Detector & PSNR$\uparrow$ & LPIPS$\downarrow$ & FVD$\downarrow$ & IoU$\uparrow$ & PSNR$\uparrow$ & LPIPS$\downarrow$ & FVD$\downarrow$ & IoU$\uparrow$ \\
\midrule
Heuristic & 31.07 & \underline{0.035} & \underline{20.67} & \underline{0.177} & 32.86 & \underline{0.058} & 34.63 & \underline{0.334} \\
Grounded-SAM & 30.60 & 0.045 & 24.79 & 0.097 & 29.15 & 0.251 & 95.35 & 0.119 \\
Grounded-SAM2 & 30.14 & 0.050 & 45.36 & 0.095 & 26.84 & 0.326 & 121.62 & 0.128 \\
TAO~\cite{tao} & \underline{32.12} & 0.038 & 23.78 & 0.117 & \underline{35.59} & 0.103 & \underline{25.23} & 0.257 \\
\rowcolor{gray!12}
Motion-gated (ours) & \textbf{33.19} & \textbf{0.032} & \textbf{20.17} & \textbf{0.205} & \textbf{39.74} & \textbf{0.047} & \textbf{20.47} & \textbf{0.343} \\
\bottomrule
\end{tabular}%
}
\end{table}

%% file: tables/tab_efficiency.tex
\begin{table}[tbh]
\centering
\caption{Inference cost on CUHK Avenue per 16-frame clip on an RTX A5500, under the same motion-gated masks. Clip medians, warm-up excluded, peak memory in GB.}
\label{tab:efficiency}
\resizebox{\columnwidth}{!}{%
\begin{tabular}{l ccc}
\toprule
 & Sec/clip & Sec/frame & Peak mem. \\
\midrule
Detection (Grounding DINO $+$ SAM) & 6.81 & 0.43 & \multirow{6}{*}{9.2} \\
Motion gate & 0.05 & $<$0.01 & \\
Classical & 0.21 & 0.01 & \\
Prior-anchored & 19.11 & 1.19 & \\
BG-diffusion & 20.74 & 1.30 & \\
Verifier & 3.72 & 0.23 & \\
\midrule
Grounded-SD (base) & 25.48 & 1.59 & 4.6 \\
ProPainter (gated masks) & 7.54 & 0.47 & 2.7 \\
AVR (single restorer) & 25.84 & 1.62 & 4.6 \\
\rowcolor{gray!12}
AVR (full) & 62.52 & 3.91 & 9.2 \\
\bottomrule
\end{tabular}%
}
\end{table}

%% file: tables/tab_real.tex
\begin{table}[t]
\centering
\caption{20 real ShanghaiTech test anomalies, evaluated without a reference. Flicker is measured inside each pipeline's own detected region, and ``impr.'' counts clips whose flicker decreased.}
\label{tab:real}
\scriptsize
\setlength{\tabcolsep}{3.4pt}
\renewcommand{\arraystretch}{1.0}
\resizebox{\columnwidth}{!}{%
\begin{tabular}{l cc cc c c}
\toprule
 & \multicolumn{2}{c}{Flicker$\downarrow$} & \multicolumn{2}{c}{Warp Err.$\downarrow$} & & \\
\cmidrule(lr){2-3} \cmidrule(lr){4-5}
Pipeline & before & after & before & after & RD$\uparrow$ & impr. \\
\midrule
Grounded-SD (base) & 12.77 & 27.71 & 0.95 & 3.22 & $-$0.513 & 7/20 \\
\rowcolor{gray!12}
AVR (Ours) & 21.12 & \textbf{3.32} & 0.95 & \textbf{0.84} & $\bm{+}$\textbf{0.721} & \textbf{20/20} \\
\bottomrule
\end{tabular}%
}
\end{table}